\documentclass[letterpaper, 10 pt, conference]{ieeeconf}

\IEEEoverridecommandlockouts                              

\usepackage{amsthm}
\usepackage[utf8]{inputenc}
\usepackage{amsmath,amssymb}
\usepackage{booktabs}
\usepackage[font=small]{subcaption}
\usepackage{numprint}
\usepackage{csvsimple}
\usepackage{comment}
\usepackage{amsfonts}
\usepackage{subcaption}
\usepackage{cite}
\usepackage{multirow}
\usepackage{eso-pic}

\newtheorem{theorem}{Theorem} 
\newtheorem{corollary}{Corollary}

\usepackage[linesnumbered,ruled,noend]{algorithm2e}
\usepackage{algpseudocode}
\usepackage{graphicx}
\usepackage{xcolor}

\newcommand{\ml}[1]{\textcolor{red}{[ML: #1]}}
\newcommand{\at}[1]{\textcolor{purple}{[AT: #1]}}

\newcommand{\reals}{\mathbb{R}}

\newcommand{\W}{\mathcal{W}}
\newcommand{\X}{\mathcal{X}}

\newcommand{\B}{\mathcal{B}}
\newcommand{\N}{\mathcal{N}}
\newcommand{\Y}{\mathcal{Y}}
\newcommand{\proj}[1]{\textsc{proj}_{#1}}
\newcommand{\safe}{\text{safe}}
\newcommand{\coll}{\text{coll}}
\newcommand{\expBelief}{\textbf{b}}

\title{\LARGE \bf
Chance-Constrained Multi-Robot Motion Planning under \\ Gaussian Uncertainties
}

\author{Anne Theurkauf$^*$, Justin Kottinger$^*$, Nisar Ahmed, and Morteza Lahijanian
\thanks{This work was supported by NASA STTR award 80NSSC20C0314.}
\thanks{$^*$A. Theurkauf and J. Kottinger have equal contributions in this work.}
\thanks{Authors are with the department of Aerospace Engineering Sciences at the University of Colorado Boulder, CO, USA
        {\tt\small \{\textit{firstname}.\textit{lastname}\}@colorado.edu}}
}

\begin{document}

\AddToShipoutPictureBG*{%
  \AtPageUpperLeft{%
    \hspace{14.5cm}%
    \raisebox{-1.5cm}{%
      \makebox[0pt][r]{Submitted to IEEE Robotics and Automation Letters}}}}

\maketitle
\thispagestyle{plain}
\pagestyle{plain}

\begin{abstract}
We consider a chance-constrained multi-robot motion planning problem in the presence of Gaussian motion and sensor noise. Our proposed algorithm, CC-K-CBS, leverages the scalability of kinodynamic conflict-based search (K-CBS) in conjunction with the efficiency of Gaussian belief trees as used in the Belief-$\mathcal{A}$ framework, and inherits the completeness guarantees of Belief-$\mathcal{A}$'s low-level sampling-based planner. We also develop three different methods for robot-robot probabilistic collision checking, which trade off computation with accuracy. Our algorithm generates motion plans driving each robot from its initial to goal state while accounting for uncertainty evolution with chance-constrained safety guarantees. Benchmarks compare computation time to conservatism of the collision checkers, in addition to characterizing the performance of the planner as a whole.  Results show that  CC-K-CBS scales up to 30 robots.
\end{abstract}

\section{INTRODUCTION}
    \label{sec:intro}
    
Robotic teams with diverse capabilities provide a powerful tool for efficiently accomplishing complex tasks. 
Applications extend from planetary exploration to automated warehouses and search and rescue missions.  However, planning for many robots is a challenging problem.
A major difficulty of \textit{multi-robot motion planning} (MRMP) lies in the size of the planning space, which grows exponentially with the number of robots. Planning must also respect the robots' kinodynamic constraints and avoid robot-robot and robot-obstacle collision, 
a difficult problem even for a single robot.
Additional complexities arise for realistic robots due to uncertainty in both motion and sensing, which the planner needs to account for in collision checking. 
In this work, we focus on the uncertain MRMP problem, and develop a scalable algorithm that guarantees collision avoidance and reaching goal with completeness properties.

Extensive work has considered MRMP problem in deterministic settings.
Proposed methods 
are either centralized (coupled)~\cite{Lavalle98rapidly-exploringrandom,wagner2015subdimensional,shome2020drrt,Kottinger:ICRA:2021,sucan2011sampling} or decentralized (decoupled)~\cite{gammell2014bit,tang2018complete,Le:ICAPS:2017,Le:RAL:2019}. 
Centralized planners compose all robots' states together into a single meta-robot state, to which single-robot planning methods with completeness guarantees can then be applied. 
However, the exponential increase in the state-space dimension makes centralized approaches scale poorly. Decentralized methods address this by planning for each robot individually and considering system collisions separately. These algorithms drastically cut computation time but often sacrifice completeness guarantees. For example, \cite{vandenBerg2005_prioritized} 
plans for one robot at a time, where each robot is constrained by the paths of all previously-planned robots. While fast, this method is provably incomplete~\cite{lavalle2006planning}. 
One recent contribution to scalable MRMP is Kinodynamic Conflict-Based Search (K-CBS), \cite{Kotting2022_KCBS}. K-CBS pairs a low-level kinodynamic motion planner for individual robots with a high-level tree search over the space of possible plans. 
K-CBS offers scalability and includes a merging technique to uphold probabilistic completeness properties of the low-level motion planner. Although a great step forward, none of these approaches are robust to state and measurement uncertainty, which are inherent to realistic robotic systems.

\begin{figure}
    \centering
    \includegraphics[width=0.26\textwidth]{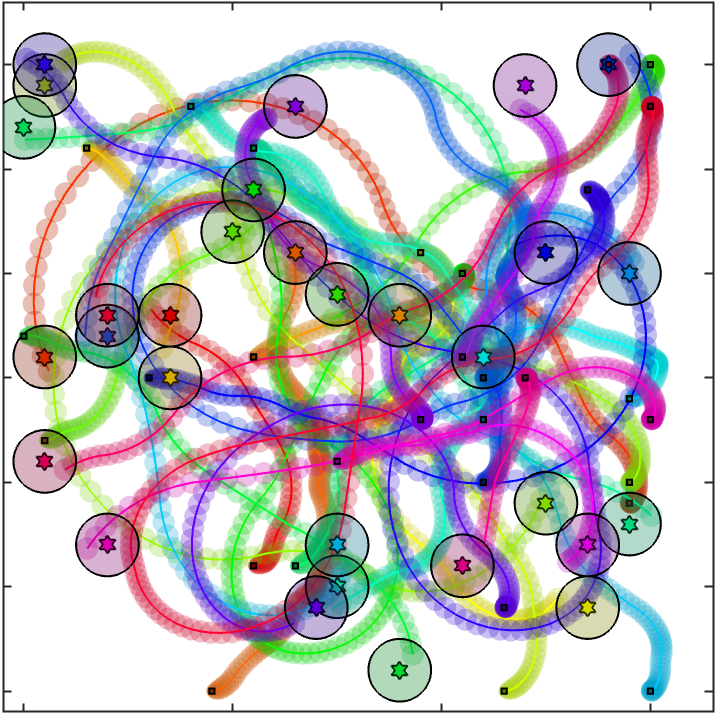}
    \caption{Sample plans for 30 robots with unicycle dynamics under motion and sensing uncertainties. Ellipses are $95\%$ safety contours. Star-centered circles are goal regions.}
    \label{fig:unicycle_SampleTraj_big}
    \vspace{-1mm}
\end{figure}

Motion planning under Gaussian uncertainty has largely been studied for single-robot systems. The most efficient algorithms are sampling-based planners, which are analogous to their deterministic counterparts, e.g., the feedback-based information roadmap \cite{Agha-Mohammadi2014_FIRM} 
and rapidly-exploring random belief trees \cite{Bry2011_BeliefProp}. 
The belief-$\mathcal{A}$ framework introduced in \cite{Ho2022_GBT} provides a generic structure for adapting any sampling-based kinodynamic planner $\mathcal{A}$ into a belief space planner. It shows that fast planning is possible by using chance constraint formulation of collision avoidance.
Although these approaches can solve uncertain motion planning, they cannot handle uncertain MRMP out of the box. 

One approach to uncertain MRMP is to abstract the problem to a multi-agent Markov Decision Process (MMDP), which can be reduced to single-agent MDPs with a joint action space \cite{Boutilier1996_MMDP}. This relies on abstraction of the robot dynamics to a discretized transition probability matrix, which is difficult and scales poorly to many robots. Additionally, MMDPs solely account for uncertain dynamics, whereas measurement noise requires using a Partially Observable MDP (POMDP), which can be very difficult to solve and also scales poorly.  Other uncertain MRMP approaches use \emph{online} distributed planning, e.g., \cite{Patwardhan2023_DistMRMP,AlonsoMora2018_CoopCollAvoid,Senbaslar2019_MultiRobotDistributed,VanParys2016_OnlineDistMultiVehicle}, where each robot plans over a short horizon to account for nearby robots and obstacles. 
These methods perform well in unknown environments. However, neither POMDP approaches nor online distributed methods provide any formal performance or completeness guarantees. 

 

We address this gap by presenting a scalable MRMP algorithm for robots operating with Gaussian state and measurement noise.
We combine 
the scalability of K-CBS with the fast belief-space planning of belief-$\mathcal{A}$ 
to solve MRMP for robots under Gaussian uncertainty. Our algorithm, Chance-Constrained K-CBS (CC-K-CBS), is probabilistically complete via the completeness inheritance properties of K-CBS and belief-$\mathcal{A}$~\cite{Kotting2022_KCBS,Ho2022_GBT}. 

We present three main contributions: (i) three probabilistic robot-robot collision checking algorithms of varying conservatism and computational complexity; (ii) a scalable decentralized planner that uses novelties from (i) to solve MRMP for robots operating with Gaussian noise; 
and (iii) case studies and benchmarks demonstrating the applicability of the larger MRMP framework, validity of our collision-checking methods and comparisons between the algorithms. 

\section{PROBLEM FORMULATION}
    \label{sec:problem}

\subsection{Uncertain Multi-Robot System}
We consider $N_A$ uncertain robots in a shared workspace $\W \subset \reals^w$, $w\in\{2,3\}$ with a set of static obstacles $\W_O \subset \W$. Each robot $i \in \{1,\ldots, N_A\}$ has a rigid body $\B^i \subset \W$ and uncertain dynamics
\begin{equation}
\label{eq: dynamics model}
    x_{k+1}^i=A^ix_k^i + B^i u_k^i + w^i_k, \quad w_k^i \sim \N(0,Q^i), 
\end{equation}
where $x_k^i\in\X^i \subseteq\reals^{n_i}$ and $u_k^i\in\mathcal{U}^i \subseteq \reals^{p_i}$ are the state and controls, respectively, with the associated $A^i \in\reals^{n_i\times n_i}$ and $B^i \in\reals^{n_i\times p_i}$ matrices, and $w^i_k$ is the motion (process) noise sampled from a Gaussian distribution with zero mean and covariance $Q^i \in \reals^{n_i \times n_i}$. 

Note that state $x_k^i \in \X^i$ corresponds to the origin of the body frame of robot $\B^i$ w.r.t. to the global frame. The robot body $\B^i$ is defined as a set of points in its local frame. 
Since state $x_k^i$ includes configuration (position and orientation) as well as velocity terms, the body of the robot is well-defined w.r.t. global frame by projection of the state to either the configuration space (C-space) or workspace ($\W$).
 We define operator $\proj{\mathcal{S}}: \cup_{i=1}^{N_A} \X^i \to \mathcal{S}$ to be this projection of $x_k^i \in \X^i$ to its the lower dimensional space $\mathcal{S}\subset \X^i$, e.g., $\proj{\W}(x_k^i)\subset\W$ is the set of points ($\B^i$) that robot $i$ at state $x_k^i$ occupies in the workspace.

Each robot $i \in \{1, \dots, N_A\}$ also has a sensor (measurement) model described by
\begin{equation}
\label{eq: measurement model}
    y_k^i=C^ix_k^i + v^i_k, \quad v_k^i \sim\N(0,R^i), 
\end{equation}
where $y_k^i\in\Y^i \subseteq\reals^{m_i}$ is the sensor measurement, $C^i \in\reals^{m_i \times n_i}$ is the corresponding matrix, and $v_k^i$ is the measurement noise, sampled from Gaussian distribution with zero mean and covariance $R^i$. We additionally assume the system is both controllable and observable.

\subsection{Estimation, Control, and Motion Plan}
\label{sec:Est Cont Motion Plan}
Each robot is equipped with a standard Kalman Filer (KF) and feedback controller.  The uncertainty of robot $i$ state, also known as the \textit{belief} of $x^i_k$ and denoted by $b(x^i_k)$,  is given by the KF as a Gaussian distribution at each time step $k$, i.e.,  
$x^i_k \sim b(x_k^i)=\N(\hat{x}_k^i,\Sigma_k^i)$, where $\hat{x}_k^i \in \X^i$ is the state estimate (mean), and $\Sigma_k^i \in \reals^{n_i \times n_i}$ is the covariance matrix.  

We define \emph{motion plan} for robot $i$ to be a pair $(\check{U}^i,\check{X}^i)$ of nominal controls $\check{U}^i=(\check{u}_0^i, \check{u}_1^i, \ldots, \check{u}_{T-1}^i)$ and the corresponding trajectory of nominal states $\check{X}^i=(\check{x}_0^i, \check{x}_1^i, \ldots, \check{x}_T^i)$, where $\check{X}^i$ is obtained via propagation of the dynamics in~\eqref{eq: dynamics model} under $\check{U}^i$ in the absence of process noise.  Then, robot $i$ executes motion plan $(\check{U}^i, \check{X^i})$ by using a proportional stabilizing feedback controller 
$u_k^i = \check{u}_{k}^i - K(\hat{x}_k^i - \check{x}_k^i),$ 
where $K$ is the control gain.


\subsection{Probabilistic Collisions and Task Completion}



Each robot $i$ is assigned a goal region $\X_{G}^i\subset \X^i$  in its state space. 
The workspace contains $ |\W_O| = N_O$ static obstacles and $N_A$ robots (moving obstacles).  The motion planning task is to compute a plan for every robot to reach its goal and avoid collisions with all (static and moving) obstacles.  In a deterministic setting, these objectives (reaching goal and obstacle avoidance) are binary.  In a stochastic setting, they are probabilistic, as defined below.  

The probability that robot $i$ is in the goal region at time $k$ is given by 
\begin{equation}
    \label{eq:prob goal}
    P^{i_k}_G = P(x^i_k \in \X^i_G) = \int_{\X^i_G} b(x_k^i)(s) ds,
\end{equation}
where $b(x^i_k)(s)$ is the distribution $b(x^i_k)$ evaluated at state $s \in \X^i$.
For collision probability, we consider static and moving obstacles separately.  Let $\X^i_O \subseteq \X^i$ be the set of static obstacles in the robot $i$'s state space.  It is constructed from $\W_O$ (and the limits on the velocity terms).  Then, similar to probability of goal, the probability of collision of robot $i$ at time step $k$ with the static obstacles is
\begin{equation}
    \label{eq:prob collision}
    P^{i_k}_O = P(x^i_k \in \X^i_O) = \int_{\X^i_O} b(x_k^i)(s) ds.
\end{equation}

The probability of one robot colliding with another is more difficult to define because we must account for uncertainty in both robot bodies. To describe the collision states, it is best to consider the composed space of the two robots.  That is, for robot $i$ and $j$, $i \neq j$, define the composed state as $x^{ij}_k = (x^i_k, x^j_k) \in \X^i \times \X^j = \X^{ij}$.  The set of points that correspond to the two robots colliding is described by
\begin{align*}
    \X^{ij}_\coll = \{(x^i_k, x^j_k) \in \X^{ij} \mid \proj{\W}(x^i_k) \cap \proj{\W}(x^j_k) \neq \emptyset\}. 
\end{align*}
The probability of collision is then given by
\begin{align}
    P^{ij_k}_{\coll} &= P((x^i_k,x^j_k) \in \X^{ij}_\coll) = \int_{\X^{ij}_\coll} b(x_k^{ij})(s) ds \nonumber 
    \\
    &= \int_{(s^i,s^j) \in \X^{ij}_\coll} \hspace{0 mm} b(x_k^{i})(s^i) b(x_k^{j})(s^j) ds^i ds^j,
    \label{eq:prob robot-robot collision}
\end{align}
where the equality in the second line is due to independence of the robots' uncertainties from each other.

In this work, we use a chance constraint to bound the safety (violating) probabilities.  That is, given a minimum probability of safety $p_\safe$, we want to ensure that probability of collision (with static and moving obstacles) at every time step is less than $1-p_\safe$, and the probability of ending in the goal is at least $p_\safe$.

\subsection{Chance-Constrained MRMP Problem}
Given $N_A$ robots with uncertain dynamics in \eqref{eq: dynamics model}, noisy measurements in \eqref{eq: measurement model}, and controller from Sec.~\ref{sec:Est Cont Motion Plan} in workspace $\W$ with static obstacles $\W_O$, 
a set of initial distributions $\{x_0^i = \N(\hat{x}_0^i,\Sigma^i_0)\}_{i=1}^{N_A}$, 
a set of goal regions $\{\X_G^i\}_{i=1}^{N_A}$, and a safety probability threshold $p_\safe$, for each robot $i \in \{1,\ldots,N_A\}$, compute a motion plan $(\check{U}^i, \check{X}^i) = ((\check{u}_0^i, \check{u}_1^i, \ldots, \check{u}_{T-1}^i), (\check{x}_0^i, \check{x}_1^i, \ldots, \check{x}_T^i))$ 
such that 
$\check{x}_0^i = \hat{x}_0^i$,
\begin{subequations}
    \begin{align}
        \label{eq:coll cc}
        &P^{i_k}_O + \sum_{j=1, j \neq i}^{N_A} P^{ij_k}_{\coll} \leq 1-p_\safe && \forall k\in \{1,..,T\} \\
        \label{eq: goal cc}
        &P^{i_T}_G \geq p_{\safe} &&
    \end{align}    
\end{subequations}
where probability terms $P^{i_k}_O$, $P^{ij_k}_{\coll}$, and $P^{i_T}_G$ are given in \eqref{eq:prob collision}, \eqref{eq:prob robot-robot collision}, and \eqref{eq:prob goal}, respectively.


Note that, unlike most related work, e.g., \cite{Patwardhan2023_DistMRMP,AlonsoMora2018_CoopCollAvoid,Senbaslar2019_MultiRobotDistributed,VanParys2016_OnlineDistMultiVehicle}, that solve the uncertain MRMP problem using \emph{online} (iterative) approaches with finite horizon predictions, the above problem asks for the entire plan with safety guarantees.  This implies that the plans must be computed \emph{offline} while the chance-constraint guarantees must hold at run time. The advantage to this approach is that the computed plans are ensured to be sound, i.e., hold hard guarantees on safety and reachability, before deployment, whereas the online methods may produce incorrect solutions, e.g., run into a deadlock.  
In this work, we seek an offline MRMP algorithm that guarantees soundness \emph{and} exhibits probabilistic completeness, i.e., generates a valid plan almost-surely 
if one exists.


\section{Chance-Constrained MRMP Algorithms}
    \label{sec:method}

We now describe two algorithms for Chance-Constrained MRMP (CC-MRMP). We begin by formalizing the composed belief space for $N_A$ robots with models in \eqref{eq: dynamics model} and~\eqref{eq: measurement model}. We then propose a centralized approach to solving CC-MRMP by employing an existing single-robot belief-space planning algorithm called \emph{Belief-}$\mathcal{A}$ \cite{Ho2022_GBT}.
To accommodate larger multi-robot systems, we introduce a decentralized algorithm called \emph{CC-K-CBS} which adapts the novel K-CBS algorithm to plan for systems with state and measurement uncertainty. We describe the empirical advantages and disadvantages of these methods in Sec.~\ref{sec:eval}.

\subsection{Centralized Chance-Constrained MRMP}
Centralized methods for uncertain MRMP require composing the system beliefs into a single \emph{meta-belief}, $\textbf{B}_k$, with a single dynamical constraint. We can then leverage Belief-$\mathcal{A}$ on this meta-belief. Belief-$\mathcal{A}$ provides a framework for adapting any kinodynamic sampling-based planner $\mathcal{A}$ to efficiently plan in the Gaussian belief space under chance constraints as described by \eqref{eq:coll cc} and \eqref{eq: goal cc}. See \cite{Ho2022_GBT} for details. 

For centralized CC-MRMP, because all the robots are assumed to be independent, sampling and propagation of the meta-belief can be performed independently for each robot, and the meta-agent belief reconstructed from the individual robots' \emph{expected beliefs}, $\expBelief(x_k^i)$. These individual beliefs are obtained through \cite{Bry2011_BeliefProp}, which provides a method for predicting the belief over states for a linearizable and controllable system that uses a KF for estimation. This method forecasts the belief with a priori unknown measurements, yielding an expected belief defined as:
\begin{align*}
    \label{eq: expected belief}
    \expBelief(x_k^i) &= \mathbb{E}_Y[b(x_k^i | x_0^i, y_{0:k}^i)] 
    = \int_{y_{0:k}^i} \hspace{-3mm} b(x_k^i | x_0^i, y_{0:k}^i) pr(y_{0:k}^i)dy.
\end{align*}
We use this expected belief for offline motion planning and evaluation of chance constraints in \eqref{eq:coll cc} and \eqref{eq: goal cc}. For a given initial belief $b(x_0^i)$ and nominal trajectory, $\check{X}^i$, the expected belief is
$\expBelief(x_k^i) = \mathcal{N}(\check{x}_k^i,\Gamma_k^i)$, where $\Gamma_k^i = \Sigma_k^i+\Lambda_k^i$ can be recursively calculated as derived in \cite{Bry2011_BeliefProp}:
\begin{align}
    \Sigma_k^{i-} &= A^i \Sigma^i_{k-1}A^{iT} + Q^i, \quad \Sigma_k^i = \Sigma_k^{i-} - L_k^i C^i\Sigma_k^{i-}, \\
    \Lambda_k^i &= (A^i-B^iK^i)\Lambda_{k-1}^i(A^i-B^iK^i)^T + L_k^i C^i \Sigma_k^{i-},
\end{align}
where $\Sigma_k^i$ is the online KF uncertainty, and $\Lambda_k^i$ is covariance of the expected state estimates $\hat{x}_k^i$. Intuitively, this distribution can be thought of as the sum of the online estimation error and the uncertainty from the a priori unknown measurements made during execution. Given the expected beliefs for the agents, the meta-belief can be constructed such that $\textbf{B}_k=\mathcal{N}(\check{X}_k^{M}, \Gamma_k^M)$, where $\check{X}_k^{M}=(\check{x}_k^{iM}, ...,\check{x}_k^{N_A M})$, and $\Gamma_k^M$ is constructed as a block diagonal with elements $\Gamma_k^i, i=1,...,N_A$. 
Similarly, the validity checkers can be iterated for each robot, such that an entire meta-belief node is considered invalid if a single robot's belief is deemed invalid.

This centralized algorithm is the simplest of our two approaches, however it suffers from poor scalability with the number of agents, just as centralized deterministic planners do. To this end, we introduce a decentralized planner. 

\subsection{Decentralized Chance-Constrained K-CBS}
\begin{algorithm}[t]
\KwIn{$\W$, $\{x_0^i\}_{i=1}^{N_A}$, N, B, $p_{\safe}$}
\KwOut{Valid motion plans $\{(\check{U}^i, \check{X}^i)\}_{i=1}^{N_A}$}
$Q, n_0, \leftarrow \emptyset$;
$n_0$.plan $\leftarrow$ \underline{\textsc{getInitPlan}}$()$;$Q$.add($n_0$)\label{ln:rootNode}\;
\While{solution not found \label{ln:mainBegin}}
{
    $c \leftarrow$ $Q$.top()\label{ln:selectKCBS}\;
    \eIf{$c$.plan is empty \label{ln:retry}}
    {
        $Q$.pop(); \underline{\textsc{replan}}(N, $c$);
        $Q$.add($c$)\label{ln:retryAdd}\;
    }
    {
        K $\leftarrow$ \underline{\textsc{validatePlan}}($c$.plan, $p_{\safe}$)\label{ln:validate}\;
        \uIf{K is empty\label{ln:isKempty}}
        {
            \KwRet{$c$.plan\label{ln:retPlan}}
        }
        \uElseIf{shouldMerge(K, B)\label{ln:shouldMerge}}
        {
            \KwRet{mrg\&rstrt($\W$, $\{x_0^i\}_{i=1}^{N_A}$, N, B, K)}
        }
        \Else
        {
            $Q$.pop()\label{ln:branchPop}\;
            \For{every robot $i$ in K\label{ln:branchForLoop}}
            {
                $c_{new}.\mathcal{C} \leftarrow$ $c.\mathcal{C} \: \cup$ \underline{\textsc{createCSTR}}(K, $i$)\label{ln:create constraint}\;
                \underline{\textsc{replan}}(N, $c_{new}$); $Q$.add($c_{new}$)\label{ln:branchAdd}\; 
            }
        }
    }
    \label{ln:mainEnd}
}
\caption{Chance Constrained K-CBS}
\label{alg:G-KCBS}
\end{algorithm}
\setlength{\textfloatsep}{10pt}

Our decentralized CC-MRMP algorithm is based on Kinodynamic Conflict-Based Search (K-CBS)~\cite{Kotting2022_KCBS}. K-CBS consists of a high-level conflict-tree search and a low-level motion planner. At the high-level, K-CBS tracks a constraint-tree (CT), in which each node represents a suggested plan, which might contain conflicts (i.e., collisions). At each iteration, K-CBS picks an unexplored node from the CT based on some heuristic and identifies the conflicts in the node’s corresponding plan. Then, K-CBS tries to resolve the conflicts by creating child nodes as follows: if robots $i$ and $j$ are in collision for some time duration $[t_s,t_e]$ (denoted $\langle i,j,[t_s,t_e]\rangle$), then two child nodes are created, one with the constraint that $\mathcal{B}^i$ cannot intersect with $\mathcal{B}^j$ for all $t\in[t_s,t_e]$ and the other dually for robot $j$.
Then, in each child node, a low-level sampling-based planner attempts to replan a trajectory that satisfies the set of all constraints in that branch of the CT within $N>0$ iterations. This process repeats until a non-colliding plan is found. In the case where greater than $B>0$ conflicts arise, K-CBS includes a merge and restart procedure that recursively combines a subset of the robots. We refer the reader to~\cite{Kotting2022_KCBS} for details on the original K-CBS.
We now describe CC-K-CBS (Alg.~\ref{alg:G-KCBS}), which requires several adaptations of K-CBS (underlined in Alg.~\ref{alg:G-KCBS}).

Firstly, note that in K-CBS, collision checking between robots and static obstacles are performed separately. The low-level planner keeps track of workspace static obstacles, and the CT node accounts for robot-robot collisions at the high-level.  This separation is key for scalability, but it introduces a major challenge in chance-constrained planning for uncertain robots.  We expand on this further in Sec.~\ref{sec:collisionChecking}.
Secondly, since states are now beliefs, CC-K-CBS requires a Gaussian belief low-level motion planner.  We propose to use Belief-$\mathcal{A}$ for the \textsc{replan} and \textsc{getInitPlan} procedures. Thirdly, CC-K-CBS requires a definition of constraints; we introduce \emph{belief-constraints}. This simply replaces the deterministic shape with a non-deterministic belief. Thus, given a conflict $K=\langle i,j,[t_s,t_e]\rangle$, one constraint forces robot $i$ to satisfy the chance-constraint with $b(x_k^j)$ for all $t\in[t_s,t_e]$ and dually for robot $j$. This enables \textsc{createCSTR} 
(Line~\ref{ln:create constraint}) 
to define meaningful constraints during replanning with Belief-$\mathcal{A}$. 

Lastly, \textsc{validatePlan} (Line~\ref{ln:validate}) requires an efficient way of validating the chance constraints (Equations \ref{eq:coll cc} and \ref{eq: goal cc}) while simulating prospective MRMP plans. We propose several techniques for this computation, each of which is described in detail in Sec.~\ref{sec:collisionChecking}. Every method requires comparison with the maximum allowable probability of collision $p_{coll} = 1-p_{\safe}$. 
We say \textsc{validatePlan} procedure is \emph{complete} if it can correctly evaluate the constraints \eqref{eq:coll cc} and \eqref{eq: goal cc} for all Gaussian beliefs.

The following theorem provides the sufficient conditions for CC-K-CBS to be probabilistically complete. 

\begin{theorem}
    \label{thm:completeness}
    CC-K-CBS is probabilistically complete if the underlying planner Belief-$\mathcal{A}$ is  probabilistically complete and the \textsc{validatePlan} procedure is complete.
\end{theorem}
\begin{proof}
The proof builds on Theorem 1 in \cite{Kotting2022_KCBS}, which shows that 
K-CBS inherits the probabilistic completeness property of the low-level planner. Thus, probabilistic completeness of Belief-$\mathcal{A}$ is required for CC-K-CBS. The completeness of \textsc{validatePlan} ensures that CC-K-CBS does not reject valid plans. It follows that \textsc{createCSTR} does not pass spurious constraints to the low-level planner. Therefore, under these conditions, the probabilistic completeness proof for K-CBS extends to CC-K-CBS.
\end{proof}

The completeness condition on \textsc{validatePlan} 
can be satisfied
by exact evaluation of the probabilities in \eqref{eq:prob goal}-\eqref{eq:prob robot-robot collision}. However, this operation can be computationally expensive. To achieve tractability and scalability, we introduce conservatism into \textsc{validatePlan} by over-approximating the probability of collision.
With this \textsc{validatePlan}, we can still guarantee soundness of the CC-K-CBS solutions, but completeness is affected since valid plans may be rejected.

\begin{corollary}
If \textsc{validatePlan} uses an over-approximation of the probability of collision, then CC-K-CBS is probabilistically complete up to the accuracy of the over-approximation.
\end{corollary}

Below, we introduce several methods for over-approximating the probability of collision. 

\section{Robot-Robot Collision Checking Algorithms}
    \label{sec:collisionChecking}

Recall that the probability of collision between robots $i$ and $j$, as defined in  \eqref{eq:prob robot-robot collision}, requires integration of the joint probability distribution of the two robots over set $\X^{ij}_\coll$. 
This computation is difficult because (i) $\X^{ij}_\coll$ is hard to construct, and (ii) integration of the joint probability distribution function is expensive.  However, the constraint-checking procedure must be extremely fast in sampling-based algorithms because it is called in every iteration of tree extension. In this section, we introduce three approximation methods for this integration that trade off accuracy with computation effort.  

\begin{figure}[b]
    \centering
    \begin{subfigure}{0.26\textwidth}
    \includegraphics[width=\textwidth]{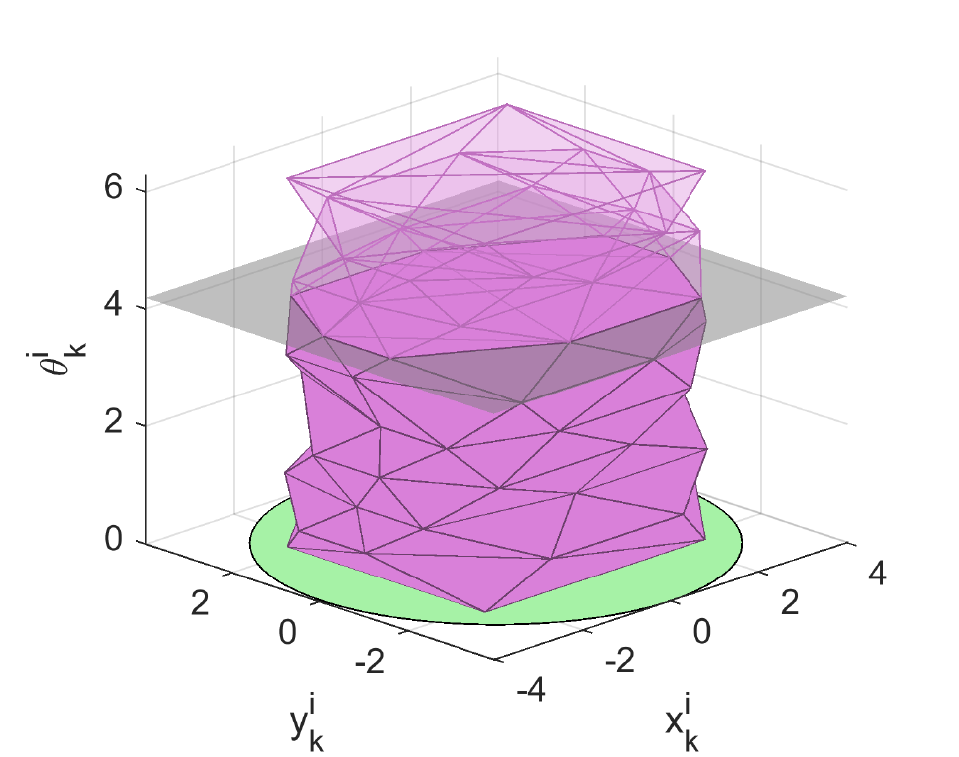}
    \vspace{-7mm}
    \caption{}
    \label{fig:CspaceIntegration}
    \end{subfigure}
    \centering
    \begin{subfigure}{0.2\textwidth}
    \includegraphics[width=\textwidth]{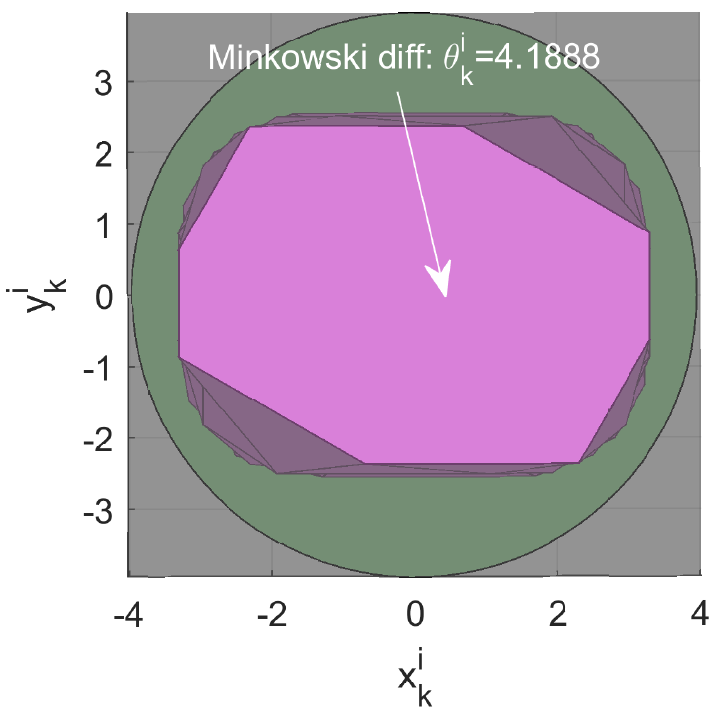}
    \vspace{-7mm}
    \caption{}
    \label{fig:CspaceIntegration_slice}
    \end{subfigure}
\caption{(a) 3D view of the c-space obstacle $\X_{coll}^{ij}$ 
for
robot $i$ with fixed robot $j$ calculated at discrete choices of $\theta_k^i$.  
(b)~Top-down view of $\X_{coll}^{ij}$.
The projection of $\X_{coll}^{ij}$ onto x-y plane is bounded by a green circle.
}
\end{figure}

For clarity of explanation, consider two deterministic (rotating and translating) robots $i$ and $j$ in a 2D workspace, with configurations $(\text{x}_k^l,\text{y}_k^l,\theta_k^l)$, $l\in \{i,j\}$. We can construct the 4D c-space obstacles for this pair by taking the Minkowski difference of the two agents for every pair of orientation angles (see the c-space construction in Fig.~\ref{fig:CspaceIntegration}, which shows the c-space in 3D for a given orientation of agent $j$). This 4D obstacle is generally difficult to construct, so deterministic collision checkers typically resort to collision checking in the workspace, e.g., by bounding the projection of this set, shown by a green disk in Fig. \ref{fig:CspaceIntegration_slice}.

Now, consider that each robot is described by a random variable (RV), so the position and orientation of their bodies are no longer deterministic but given by beliefs, such that each point in the 4D c-space maps to a probability. The exact probability of collision is the integral of the probability mass over the c-space obstacle. 
This results in the two challenges listed above: c-space obstacle construction and integral computation.
Each method 
proposed below attempts to address these core problems. 
\begin{figure}
     \centering
     \begin{subfigure}[b]{0.45\columnwidth}
         \centering
         \includegraphics[width=\textwidth]{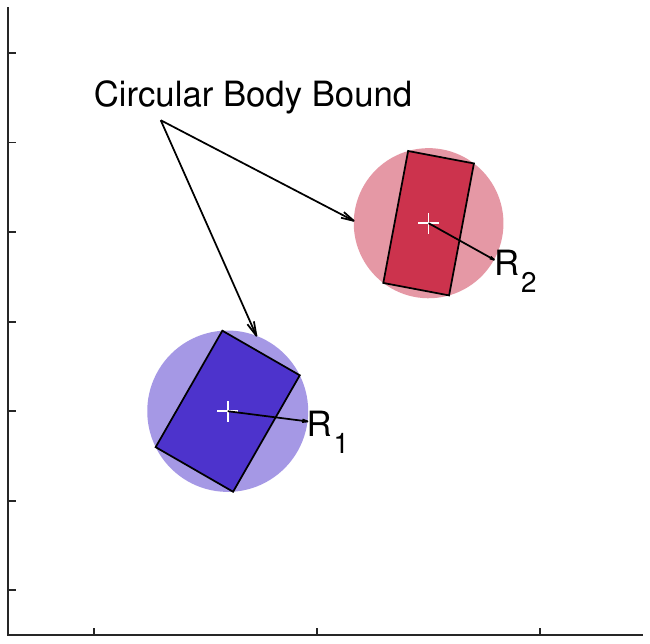}
         \vspace{-5mm}
         \caption{}
         \label{fig:diskBoundBody}
     \end{subfigure}
     \hfill
     \begin{subfigure}[b]{0.45\columnwidth}
         \centering
         \includegraphics[width=\textwidth]{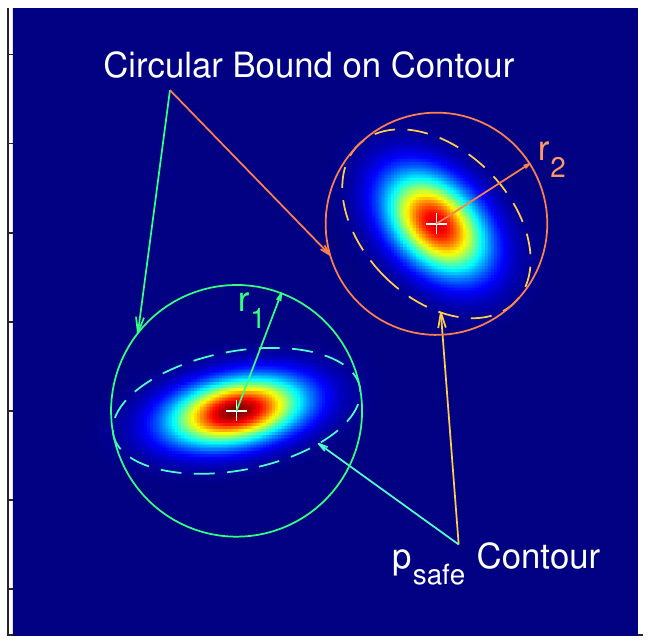}
         \vspace{-5mm}
         \caption{}
         \label{fig:AgentCircPsafe}
     \end{subfigure}
     \newline
     \begin{subfigure}[b]{0.45\columnwidth}
         \centering
         \includegraphics[width=\textwidth]{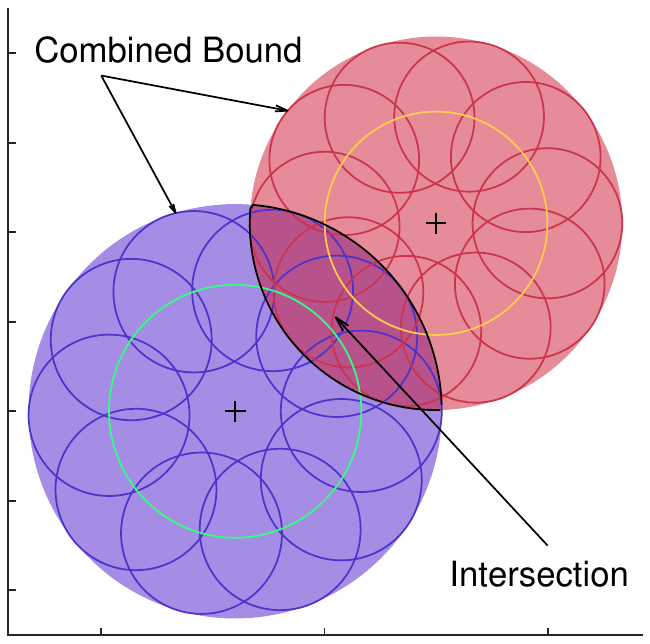}
         \vspace{-5mm}
         \caption{}
         \label{fig:diskIntersect}
     \end{subfigure}
     \hfill
     \begin{subfigure}[b]{0.45\columnwidth}
         \centering
         \includegraphics[width=\textwidth]{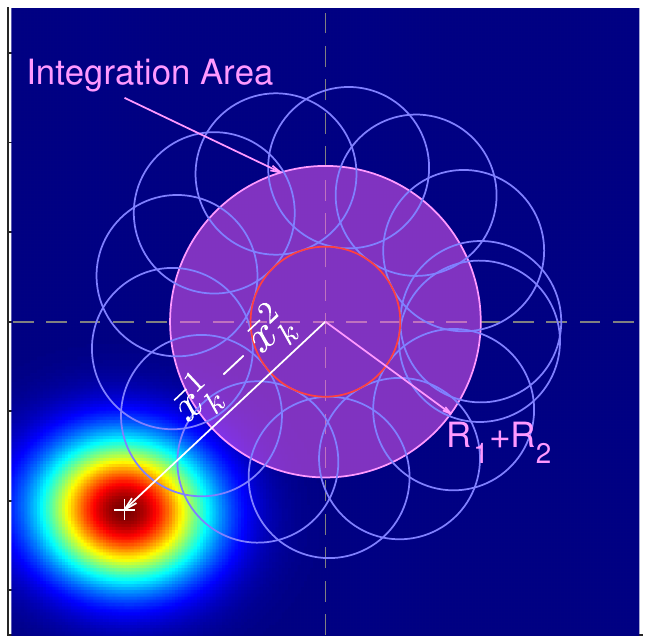}
         \vspace{-5mm}
         \caption{}
         \label{fig:CircularBodyCicularIntegration}
     \end{subfigure}
     \newline
     \begin{subfigure}[b]{0.45\columnwidth}
         \centering
         \includegraphics[width=\textwidth]{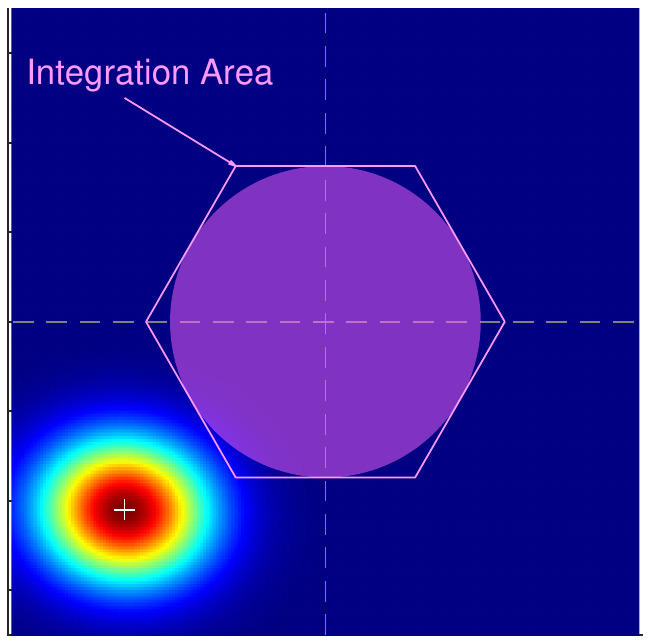}
         \vspace{-5mm}
         \caption{}
         \label{fig:CircularBodyPolygonIntegration}
     \end{subfigure}
     \hfill
     \begin{subfigure}[b]{0.45\columnwidth}
         \centering
         \includegraphics[width=\textwidth]{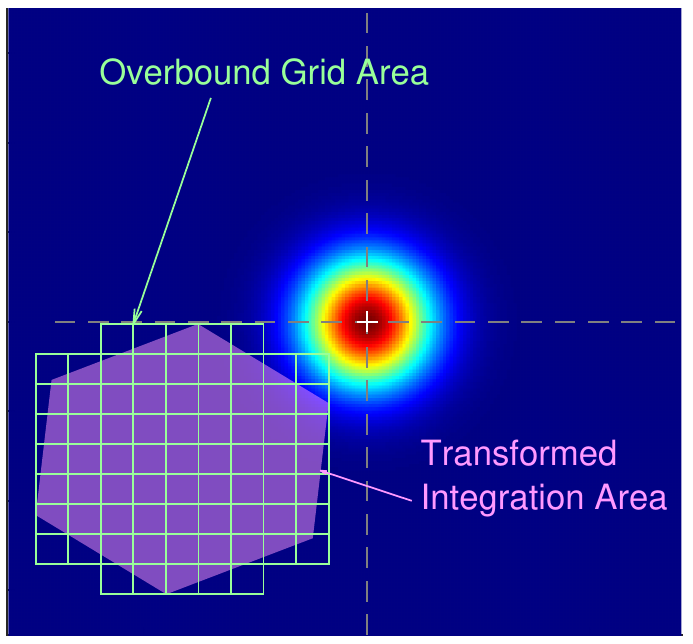}
         \vspace{-5mm}
         \caption{}
         \label{fig:CircularBodyGridIntegration}
     \end{subfigure}
     \caption{
(a) Circle bound on robots' bodies. (b) $p_\safe$ contour and corresponding circular over bound on $p_\safe$. (c) Inflated circular bounds intersect, indicating possible collision between robots. (d) Integration area for probability of collision with circular disk bound on robots' bodies. (e) Polygonal bound on circular integration area. (f) Transformed distribution and integration area, with over-approximated grid.}
\end{figure}

\subsection{Method 1 (M.1): Safety Contour}
\label{sec:Method1}
M.1 is inspired by deterministic collision checking methods, where collisions are detected by intersections of robot bodies in the workspace. We adapt this framework to uncertain collision checking by inflating bounds on the robot bodies to encapsulate $p_\safe$ probability mass.


Recall that the expected belief is $\expBelief(x_k^i)= \N(\check{x}_k^i,\Gamma_k^i)$. We define the Gaussian marginal of this distribution in the workspace as $\expBelief_\W(x_k^i)$.
Let $\bar{x}^i_k \in \W$ be the position components of $x_k^i$, and $\bar{\Gamma} \in \reals^{w\times w}$ be the corresponding covariance matrix (sub-matrix of $\Gamma^i_k$). Then $\expBelief_\W(x_k^i) = \N(\bar{x}^i_k, \bar{\Gamma}^i_k)$.
The elliptical probability contour containing $p_{\safe}$ probability mass in the workspace marginal distribution is the region $\mathcal{E}_{\safe}^i$ such that
    $\int_{\mathcal{E}_{\safe}^i} \expBelief_\W(x_k^i) dx=p_{\safe}.$
The axes of this ellipsoid are given by $a_l^i=t_{\chi}\lambda_l, \quad l=1,\ldots,w,$ 
where $\lambda_l$ is an eigenvalue of $\bar{\Gamma}_k^i$, and
$t_{\chi}$ is the inverse $\chi^2$ cumulative density function
evaluated at $p_{\safe}$ with $w$ degrees of freedom.  

Determining intersection of ellipses is difficult, so we bound the ellipsoid with a sphere, where checking for intersection only requires comparing the sum of the radii to the distance between the centers. The sphere containing this ellipsoid, $\mathcal{C}_{\safe}^i$, is defined by the radius $r^i=t_{\chi}\lambda_{max}$, where 
$\lambda_{\max}=\max\{ \lambda_l \}_{l=1}^w.$
Because $\mathcal{E}_{\safe}^i\subseteq \mathcal{C}_{\safe}^i$, it is assured that the probability mass contained by this sphere is greater than or equal to $p_{\safe}$, i.e., 
    $\int_{\mathcal{C}_{\safe}^i} \expBelief_\W(x_k^i) dx\geq p_{\safe}.$
An illustration of both the elliptical contour and its associated spherical bound are shown in Fig.~\ref{fig:AgentCircPsafe} for the robots in Fig~\ref{fig:diskBoundBody}.

Note that this sphere contains only the probability mass associated with $x_k^i$, the origin point on the body frame. 
It does not account for the collision probability of the other points on the robot body.
To address this, we inflate the sphere by a bounding sphere on the robot's body (shown in Fig.~\ref{fig:diskBoundBody}), such that the probability of \emph{any} point on the robot's body being within this inflated sphere is greater than or equal to $p_{\safe}$. The radius, $R_i$, of the bounding sphere on the robot's body is obtained by finding the largest possible Euclidean distance between any pair of points within a robot's body: $R_i = \frac{1}{2}\max_{x_1,x_2\in \B^i} \| x_1-x_2 \|_2$. The inflated bound, $\mathcal{C}_{bound}^i$ is a sphere of radius $r^i+R^i$. The probability that any point of the robot is outside this sphere, the complement of $\mathcal{C}_{bound}^i$, is guaranteed to be less than $1-p_{\safe}$. Thus, if any two robot's bounding spheres
do not intersect, i.e. $\mathcal{C}_{bound}^i \cap \mathcal{C}_{bound}^j = \emptyset$, the probability of collision between them is guaranteed to be less than $1-p_{\safe}$.
The inflated disks and corresponding intersection checking are shown in Fig.~\ref{fig:diskIntersect}.

This method has two main sources of conservatism. The first stems from the bounding sphere on the robot body. This is especially egregious for robots with high aspect ratios. The second source is associated with checking for intersections of the probability mass bound. This conservatively assumes all probability mass not within the sphere is in collision, so any intersection (no matter how small) is considered a possible constraint violation. Violation for this method does not imply that the actual probability of collision exceeds the constraint. Despite these limitations, this method is extremely fast. 

\subsection{Method 2: Linear Gaussian Transformation}
\label{sec:Method2}
This section introduces a class of methods that reduce conservatism by simplifying the integration itself, namely by reducing its dimensionality. This is done by introducing a new RV representing the difference between two robots' states: $x_k^{ij}=x_k^i-x_k^j$. This RV is Gaussian distributed such that $x_k^{ij}\sim\N(\check{x}_k^i-\check{x}_k^j, \Gamma_k^i+\Gamma_k^j)$, with the workspace marginal $\expBelief_\W(x_k^{ij})$. This new RV represents a vector connecting the origin points of two robots. Collisions correspond to specific realizations of $x_k^{ij}$, 
defining the required integration  area.

\subsubsection{Method 2.1 (M.2.1): Convex Polytopic Bounding}
\label{sec:Method2.1}
The difference RV somewhat simplifies the high dimensional c-space obstacle representation, shown in Fig.~\ref{fig:CspaceIntegration}, by merging the workspace dimensions for the two robots. But the orientation of the robots must still be considered. To obtain a tractable integration area, we use the same bounding spheres described in Sec.~\ref{sec:Method1}. Intersection of the bounding spheres is considered a collision state, corresponding to a spherical integration area $S_{body}^{ij}$, with radius $R^i + R^j$, over difference distribution $\expBelief_\W(x_k^{ij})$ (shown in Fig.~\ref{fig:CircularBodyCicularIntegration}). 

Because exact integration of a Gaussian distribution over a sphere is difficult, we bound $S_{body}^{ij}$ with a polytope, $\mathcal{P}_{body}^{ij}$, which can be defined by a set of $N_s$ half planes, $\mathcal{P}_{body}^{ij} = \{ x \mid c_{h,i}^Tx<d_{h,i}, \forall h\in[1,N_s]\}$, where $c_{h,i}\in\mathbb{R}^{n\times 1}$ and $d_{h,i}\in\mathbb{R}$. This reduces the robot-robot chance constraint checking problem to the same problem addressed in \cite{Luders2010_CC-RRT},  where various (over-)approximations are used to reduce the probabilistic chance constraint checking to deterministic linear constraints checking. We employ the same methods to check the chance constraints, with the distribution given by $\N(\check{x}_k^i-\check{x}_k^j, \Gamma_k^i+\Gamma_k^j)$, and polytope given by $\mathcal{P}_{body}^{ij}$.

A key approximation in this method is the allocation of the acceptable probability of collision, or risk allocation: the total acceptable probability of collision $p_{coll} = 1-p_{\safe}$ must be distributed across all the possible sources of collision. In our case, we must distribute the probability of collision for a robot among the $N_A-1$ other robots. We denote the allowable probability for robot $i$ colliding with any other robot $j$ as $p_c^{ij}$. Two proposed methods of determining $p_c^{ij}$ are described in Sec.~\ref{sec:Risk Allocation}. 

This method has two main sources of conservatism. The first occurs with the spherical bound on the body, the same bound used in M.1 (Sec.~\ref{sec:Method1}) while the second arises from the use of deterministic linear constraints for the integration. This source of conservatism scales with the number of robots and the number of half planes in the polytope bound. This method is only slightly more computationally expensive than the first method: while checking linear constraints is very computationally simple, this method must at worst check every constraint associated with each pair of robots. This yields slightly longer computation times than the M.1, which requires checking a single condition for every robot pair.

\subsubsection{Method 2.2 (M.2.2): Grid Integration}
This method follows directly from M.2.1, but with a direct integration technique, rather than a check on deterministic linear constraints. This makes this method more accurate and less conservative, but also more computationally intensive. 

We begin with the same difference distribution over $\expBelief_\W(x_k^{ij})$ 
and integration over $\mathcal{P}_{body}^{ij}$. We then perform a whitening (Mahalanobis) transformation on the distribution and $\mathcal{P}_{body}^{ij}$. This transformation decorrelates a Gaussian distribution, yielding the standard normal distribution, as well as rotating and translating $\mathcal{P}_{body}^{ij}$. We define this new integration area as $\mathcal{P}_{white}^{ij}$ (see Fig.~\ref{fig:CircularBodyGridIntegration}). 
Because we have transformed the probability distribution to a standard normal, the cumulative distribution function can be easily computed via the error function (erf); we do this over a grid. This technique is agnostic to the choice of grid cells and only requires rectangular cells that completely cover the polytope $\mathcal{P}_{white}^{ij}$, as shown in green in Fig. \ref{fig:CircularBodyGridIntegration}. The probability of collision is the sum of the probability mass contained in each grid cell, defined as $p_{poly}$. The chance constraint can then be simply checked as $p_{poly}\leq p_c^{ij}$,
where $p_c^{ij}$ again is the risk allocated to a pair of robots.




This method inherits the same conservatism associated with the spherical body bound and polytope integration area bound in M.2.1. However, the approximate grid calculation is far less conservative than the deterministic linear constraints. The grid introduces more computational cost, particularly in transforming the distribution and evaluating the probability of each grid cell. Finer grid discretization reduces conservatism, but also increases computation time.

\subsection{Risk Allocation}
\label{sec:Risk Allocation}
Risk allocation is only necessary for M.2.1 and M.2.2 and not required in M.1. Equally distributing the allowable probability of collision $p_{coll}$ across the robots and obstacles, as introduced in \cite{Luders2010_CC-RRT}, is the simplest allocation method. For robot-robot collision this corresponds to $p_c^{ij}=\frac{p_{coll}}{N_O + N_A - 1}$. However, this can be overly conservative and make it difficult to find viable motion plans. Instead, we propose allocating $p_{coll}$ more effectively, such that robots with a higher likelihood of collision are assigned a higher proportion of $p_{coll}$.

We begin by setting $p_{coll}=P_c^A+P_c^O$, where $P_c^A$ 
and
$P_c^O$ are the allowable probabilities of collision with all the robots and all the static obstacles, respectively. This naturally fits the high-level low-level division in CC-K-CBS because obstacle collision checking occurs only in the low-level planner.
Define the total volume of all the robots as $V_A$, and the volume of all the obstacles as $V_O$. Then, we set  $P_c^A=\frac{V_A p_{coll}}{V_A+V_O}$ and $P_c^O=\frac{V_Op_{coll}}{V_A+V_O}$. This allocates more of the probability of collision to the category that is more likely to cause collisions. 
We then divide $P_c^A$ based on the distance between robots, with robots that are closer together receiving a larger portion of $P_c^A$. Let $d_k^{j}=\|\bar{x}_k^i-\bar{x}_k^j\|$ be the workspace distance between robot $i$'s and $j$'s means. Then, the allowable probability of collision for robot $i$ with $j$ is $p_c^{ij} = \frac{\alpha}{d_k^j}P_c^A$, where 
$\alpha =1/ (\sum_j^{N_A-1} d_k^{j})$ 
is the normalizing factor.

\section{EVALUATIONS}
    \label{sec:eval}
    
We evaluate our algorithms in several benchmarks. We first independently test the robot-robot collision checking algorithms in Sec.~\ref{sec:collisionChecking}, allowing us to isolate the conservatism and computational complexity of the robot-robot collision checkers. We then analyze the full algorithms from Sec.~\ref{sec:method} with different collision checking methods.
Our implementation of CC-K-CBS is publicly available~\cite{kcbs-code}. It uses belief-SST as the low-level planner in Open Motion Planning Library (OMPL) \cite{sucan2012the-open-motion-planning-library}. 
The benchmarks were performed on AMD 4.5 GHz CPU and 64 GB of RAM. 

\begin{figure}
    \centering
    \begin{subfigure}{0.23\textwidth}
    \includegraphics[width=\textwidth]{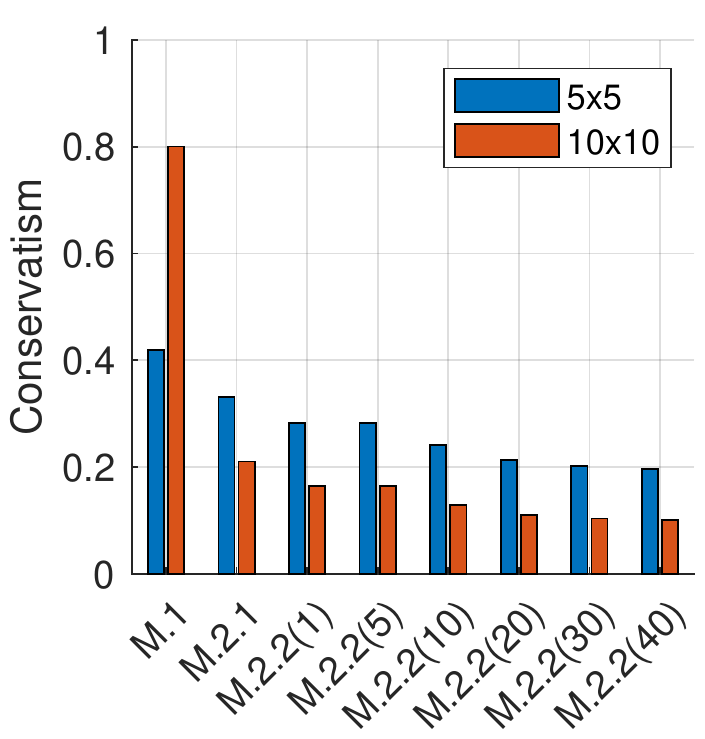}
    \vspace{-5mm}
    \caption{}
    \label{fig:conservatism}
    \end{subfigure}
    \hfill
    \begin{subfigure}{0.23\textwidth}
    \includegraphics[width=\textwidth]{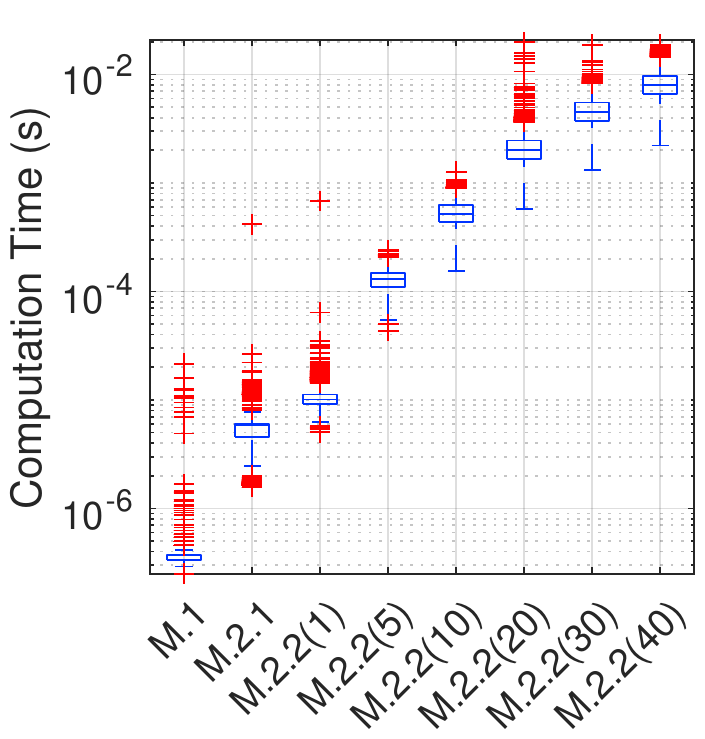}
    \vspace{-5mm}
    \caption{}
    \label{fig:compTime_combined}
    \end{subfigure}
    \hfill

\caption{(a) Plot of the conservatism $C_{z}$ for each method in each environment. (b) Box and whiskers plots of the log(computation time) for each method.}
\end{figure}

We consider square-shaped robots of width 0.25 with two types of dynamics: 2D robotic system taken from~\cite{Bry2011_BeliefProp}; $x_{k+1}^i=x_k^i+u_k^i+w_k^i, w_k^i \sim \mathcal{N}(0,0.1^2I)$, and 2nd-order unicycle; $\dot{\text{x}}^i= \text{v} \cos\theta, \dot{\text{y}}^i= \text{v} \sin \theta, \dot{\theta}=\omega, \dot{\text{v}}=a$, where $\omega$ and $a$ are control inputs. We use the feedback linearization 
scheme described in \cite{DeLuca2000_feedbackLin} 
to obtain a linear model, 
which we convert to discrete time with additive noise. We assume both systems are fully observable with measurement model in \eqref{eq: measurement model} with $C^i=I$ and $R^i=0.1^2I$.

\textbf{Robot-Robot Collision Checker Benchmark: }
We characterize the robot-robot collision checking algorithms based on computation time and conservatism. Conservatism impacts the planners ability to find paths in more difficult, i.e. cluttered, situations. For a given set of beliefs for two robots, we define the conservatism of method $z$ as $C_{z} = R_{z} - R_{MC}$, where $R_{z}$ is the rejection rate obtained by method $z$, and $R_{MC}$ is the rejection rate obtained by a Monte Carlo evaluation of the probability of collision.

We sampled 2D beliefs in two difference spaces, 5$\times$5 and 10$\times$10, with $p_\safe=0.95$. For M.2.2 (the gridded cdf method), we chose 6 different grid discretizations, denoted by M.2.2($d$). For each discretization, the maximum range of $\mathcal{P}_{whit}^{ij}$ over each axis, $D_w$, was divided equally, such that the discretization step size was $D_w/d$. Note that each of the proposed methods decreases in conservatism, shown in Fig.~\ref{fig:conservatism}, and increases in computation time, shown in Fig.~\ref{fig:compTime_combined}. The 5$\times$5 space is the more `cluttered', with more beliefs in close proximity, which in turn results in higher collision probabilities. Note that the 5$\times$5 space generally results in lower conservatism. This is because for the relatively high choice of $p_\safe$, a high proportion of the sampled beliefs truly violated the chance constraint, making it less conservative compared to a scenario where more beliefs are valid.

\begin{figure}[t]
    \centering
    \begin{subfigure}[b]{0.49\columnwidth}
         \centering
         \includegraphics[width=\textwidth]{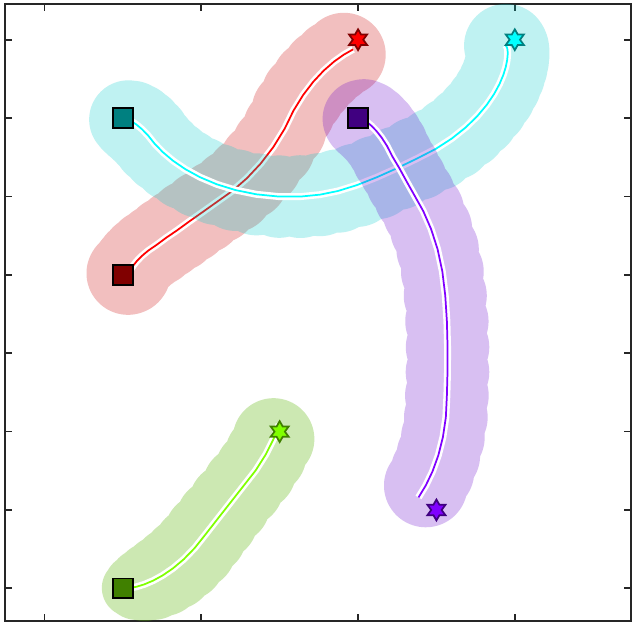}
         \caption{Env8}
         \label{fig:unicycle_SampleTraj_8x8}
     \end{subfigure}
    \begin{subfigure}[b]{0.485\columnwidth}
         \centering
         \includegraphics[width=\textwidth]{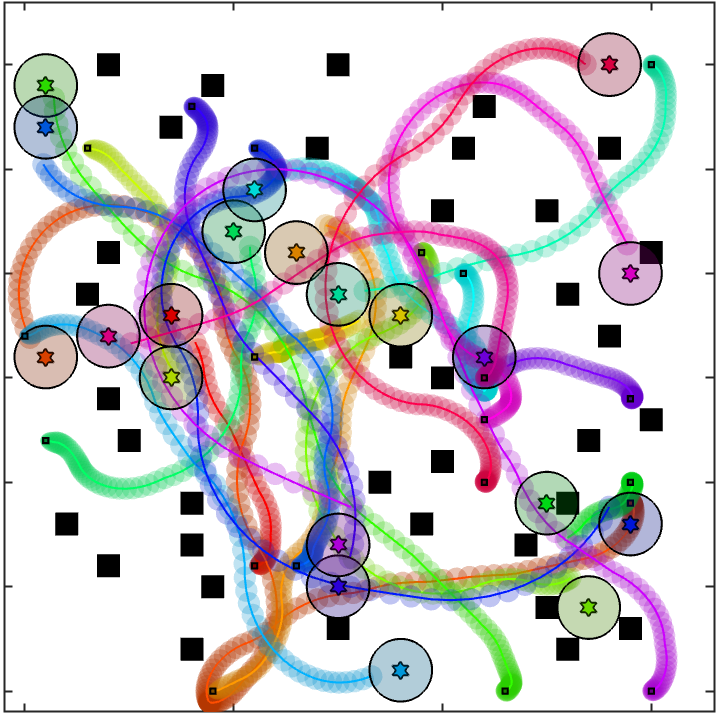}
         \caption{Env32Obs}
         \label{fig:unicycle_SampleTraj_32x32obs}
     \end{subfigure}
     \hfill
     \hspace{-2mm}
     \caption{(a) Sample plan for 4 robots with 2nd-order unicycle dynamics in Env8. Ellipses are $95\%$ probability contours. (b) Sample plan for 20 robots with 2nd-order unicycle dynamics in Env32Obs. Ellipses are $95\%$ probability contours.}
     \label{fig:envs}
\end{figure}



\begin{table}[t]
\caption{CC-K-CBS benchmarks for Env8.}
    \centering
    \scalebox{0.82}{
    \begin{tabular}{|c|c|c|c|c|c|c|c|}
         \hline
         \multicolumn{8}{|c|}{2D Linear} \\ \hline
        $N_A$ & Metric & M.1 & M.2.1 & M.2.1* & M.2.2(2) & M.2.2(5) & M.2.2(10) \\ \hline 
        \multirow{3}{*}{2} &  Succ. Rate &  1.00&   1.00 &   1.00 &   1.00 &   1.00 &   1.00\\ 
         &  Time (s) &    1.08 &   \textbf{0.68} &   1.04 &   1.11 &   0.69 &   0.74 \\ \hline 
        \multirow{3}{*}{4} &  Succ. Rate &  0.00&   1.00 &   1.00 &   1.00 &   1.00 &   1.00\\ 
         &  Time (s) &     - &  18.68 &  35.03 &   4.19 &   \textbf{3.97} &   7.61 \\ \hline 
        \multirow{3}{*}{6} &  Succ. Rate &  -&   0.00 &   0.12 &   1.00 &   1.00 &   1.00\\ 
         &  Time (s) &     - &    - &  66.68 &   \textbf{4.44} &   6.15 &  12.72 \\ \hline 
        \multirow{3}{*}{8} &  Succ. Rate &  -&   - &   0.00 &   1.00 &   1.00 &   0.86\\ 
         &  Time (s) &     - &    - &    - &  \textbf{15.38} &  26.54 &  59.08 \\ \hline  
         \multicolumn{8}{|c|}{Unicycle} \\ \hline
        \multirow{3}{*}{2} &  Succ. Rate &  1.00&   1.00 &   1.00 &   1.00 &   1.00 &   1.00\\ 
         &  Time (s) &    0.95 &   0.63 &   0.61 &   \textbf{0.37} &   0.57 &   0.83 \\ \hline 
        \multirow{3}{*}{4} &  Succ. Rate &  0.00&   0.96 &   1.00 &   1.00 &   1.00 &   1.00\\ 
         &  Time (s) &     - &  17.67 &   9.81 &   \textbf{2.59} &   2.86 &   5.24 \\ \hline 
        \multirow{3}{*}{6} &  Succ. Rate & -&   0.00 &   0.00 &   1.00 &   1.00 &   1.00\\ 
         &  Time (s) &     - &    - &    - &   \textbf{4.79} &   5.95 &  10.40 \\ \hline 
        \multirow{3}{*}{8} &  Succ. Rate &  -&   - &   - &   1.00 &   0.98 &   0.96\\ 
         &  Time (s) &     - &    - &    - &  \textbf{18.27} &  25.52 &  45.73 \\ \hline 
    \end{tabular}
    }
    \label{tab:benchCCKCBS}
\end{table}

\textbf{Planners Benchmark: }
We use benchmarking to evaluate the proposed planners under each of the collision-checking methods in Sec.~\ref{sec:collisionChecking}. We additionally implement the adaptive risk allocation on M.2.1, denoted by M.2.1*. All other methods use equal allocation. 
Benchmarking was performed using 50 runs for each algorithm with a maximum planning time of 3 minutes and $P_\safe = 0.9$. We examine 3 environments: a small 8$\times$8 (Env8), large 32$\times$32 with 50 random obstacles (Env32Obs), and large 32$\times$32 without obstacles (Env32) shown in Figs.~\ref{fig:unicycle_SampleTraj_8x8}, \ref{fig:unicycle_SampleTraj_32x32obs}, and \ref{fig:unicycle_SampleTraj_big}, respectively.
The considered metrics are \emph{success rate} (the ratio of the successful planning instances within 3 minutes to all instances) and \emph{time} (computation time of the successful runs).

\textit{Env8: } This small environment allows us to study how the different collision checking methods perform in a cluttered environment with the 2D system. The results are shown in Table~\ref{tab:benchCCKCBS}. Conservatism strongly impacts the planner's ability to find paths with additional robots. All methods reliably find plans for 2 robots, however as the conservatism of the method increases, the success rate decreases when more agents are added. In this cluttered environment robot paths are necessarily close together, so more conservative methods are less likely to find valid paths.  Specifically, M.1 performs poorly and M.2.2 has the best performance.


\textit{Env32Obs: } The large environment is far less cluttered, allowing for far more agents and better insight into the scalability of the CC-K-CBS algorithm. The results for Env8 are based on the robot-obstacle collision checker from \cite{Luders2010_CC-RRT}. However, this is very conservative, especially for many obstacles.  CC-K-CBS with this method was unable to find paths in Env32Obs for even 2 robots. To scale to a larger number of robots, we made the very straightforward adaptation to M.1 to check intersection of the bound $\mathcal{C}_{bound}^i$ with obstacles. 
As shown in Fig~\ref{fig:benchmark 30robots}, CC-K-CBS is able to scale to $20$ robots with $68\%$ success rate and runtime of $120.7\pm5.45$ seconds for the 2nd-order unicycle dynamics. A sample plan is shown in Fig.~\ref{fig:unicycle_SampleTraj_32x32obs}.  This illustrates that for large number of robots, M.1 has the best performance because it is not affected by risk allocation, whereas the allocated risk gets prohibitively smaller as the number of robots increases for the other methods, making it difficult to find valid plans.

We also performed benchmarking on the centralized approach with the simpler 2D system in Env32Obs, using collision checking M.1. 
This algorithm was able to plan for 2 robots with $100\%$ success rate, but the success rate dropped to $8\%$ for 3 robots, illustrating the classical scalability issue with centralized approaches. 

\textit{Env32: } To further test the scalability of CC-K-CBS, we increased the number of agents in the large open environment Env32. 
As shown in Fig~\ref{fig:benchmark 30robots}, CC-K-CBS is able to to scale to $28$ 2nd-order unicycle robots with $56\%$ success rate and $129\pm4.97$ seconds computation time. 
A sample plan is shown in Fig.~\ref{fig:unicycle_SampleTraj_big}.  CC-K-CBS is also able to plan for 30 robots but at much lower success rate of $18\%$.

\begin{figure}
    \centering
    \includegraphics[width=0.86\columnwidth]{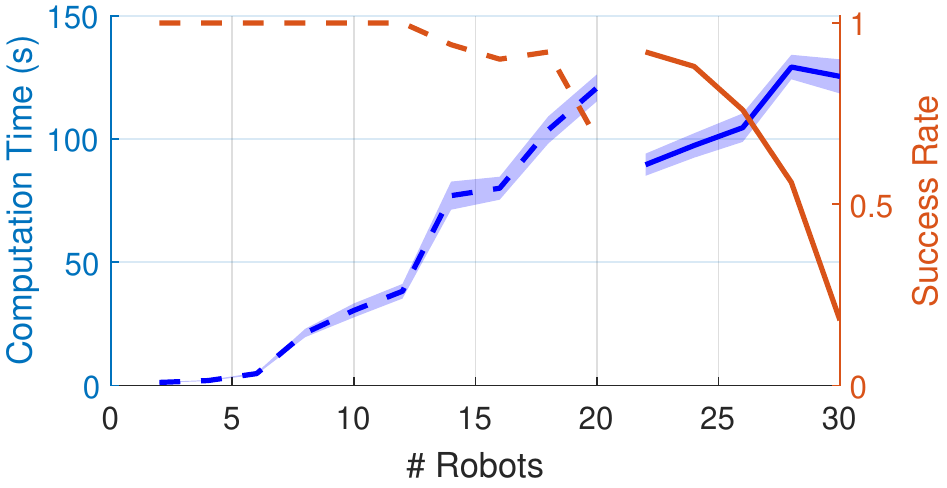}
    \caption{Benchmark results for CC-K-CBS on varying number of 2nd-oder unicycle robots. For \# robots $\leq 20$, Env32Obs is used (dashed line), and for \# robots $\geq 22$, Env32 is used (solid line). The discontinuity is due to removing obstacles.
    }
    \label{fig:benchmark 30robots}
\end{figure}

We validated the robustness of the motion plans using Monte Carlo simulation. We randomly selected motion plans, and with 500 simulations for each robot and observed no more than $2\%$ collision probability over the entire trajectory, illustrating the plans are indeed robust to uncertainty.

\section{CONCLUSION}
    \label{sec:conclusion}
    In this work, we consider the MRMP problem under Gaussian uncertainty and introduce CC-K-CBS, which generates motion plans for each robot with chance-constraint guarantees.
This algorithm is based on extending K-CBS to accommodate belief space planning. 
Our main contribution includes three methods for fast evaluation of collision probabilities between robots. 
This offline algorithm successfully generated motion plans across a variety of scenarios and scaled to 30 robots.
Future directions could encompass other types of noise and nonlinear dynamics as well as new collision checking methods that could further 
improve computation time and conservatism.


\bibliographystyle{IEEEtran}
\bibliography{refs}

\end{document}